\documentclass[lettersize,journal]{IEEEtran}
\usepackage{algorithm}
\usepackage{array}
\usepackage[caption=false,font=normalsize,labelfont=sf,textfont=sf]{subfig}
\usepackage{textcomp}
\usepackage{stfloats}
\usepackage{url}
\usepackage{verbatim}
\usepackage{balance}
\usepackage{cite}
\IEEEoverridecommandlockouts   
\usepackage{graphics}
\usepackage{multirow}
\usepackage{multicol}
\usepackage{graphicx}
\usepackage{adjustbox}
\usepackage{hyperref}
\usepackage{breqn}
\usepackage{siunitx}
\usepackage{hhline}
\usepackage{todonotes}
\usepackage{comment}
\usepackage{csquotes}          
\usepackage{amsmath,amssymb,amsthm,amsfonts}
\usepackage{algpseudocode}
\usepackage[left]{lineno}
\newtheorem{theorem}{Theorem}
\newtheorem{remark}[theorem]{Remark}
\newtheorem{lemma}[theorem]{Lemma}
\newtheorem{corollary}[theorem]{Corollary}
\theoremstyle{definition}

\newcommand{\cat}[0]{\mathcal{C}}

\usepackage{todonotes}
\newcommand{\rev}[1]{\textcolor{black}{#1}}

\title{\LARGE \bf Efficient variable-length hanging tether parameterization for marsupial robot planning in 3D environments*}

\author{S. Mart\'inez-Rozas$^{1}$, D. Alejo$^{2}$, F. Caballero$^{2}$, L. Merino$^{2}$, M.A. Pérez-Cutiño$^{3}$, \\ F. Rodríguez$^{3}$, V. Sánchez-Canales$^{3}$,
I. Ventura$^{3}$ and J.M. Díaz-Báñez$^{3}$
\thanks{*This work was partially supported by the grants: 1) INSERTION PID2021-127648OB-C31 and NORDIC TED2021-132476B-I00, funded by MCIN/AEI/ 10.13039/501100011033 and the “European Union NextGenerationEU/PRTR”, and 2) PID2020-114154RB-I00 and TED2021-129182B-I00, funded by
MCIN/AEI/10.13039/501100011033 and the “European Union NextGenerationEU/PRTR.}
\thanks{$^{1}$S. Mart\'inez-Rozas is with Universidad de Antofagasta, Antofagasta, Chile. Email: {\tt\small simon.martinez$@$uantof.cl}}
\thanks{$^{2}$D. Alejo, F. Caballero and L. Merino are with Service Robotics Laboratory, Universidad Pablo de Olavide, Seville, Spain. Email: {\tt\small daletei$@$upo.es}, {\tt\small fcaballero$@$us.es}, {\tt\small lmercab$@$upo.es}}
\thanks{$^{3}$ F. Rodríguez, M.A. Pérez-Cutiño, V. Sánchez, I. Ventura and J.M. Díaz-Báñez is with Department of Applied Mathematics II, Universidad de Sevilla, Spain. Email: {\tt\small frodriguex@us.es}, {\tt\small migpercut$@$alum.us.es}, {\tt\small iventura$@$us.es}, {\tt\small vscanales$@$us.es}, {\tt\small dbanez@us.es}} 
}

\begin{document}
\maketitle
\thispagestyle{empty}
\pagestyle{empty}

\begin{abstract}

This paper presents a novel approach to efficiently parameterize and model the state of a hanging tether for path and trajectory planning of a UGV tied to a UAV in a marsupial configuration. Most implementations in the state of the art assume a taut tether or make use of the catenary curve to model the shape of the hanging tether. The catenary model is complex to compute and must be instantiated thousands of times during the planning process, becoming a time-consuming task, while the taut tether assumption simplifies the problem, but might overly restrict the movement of the platforms. In order to accelerate the planning process, this paper proposes defining an analytical model to efficiently compute the hanging tether state, and a method to get a tether state parameterization free of collisions. We exploit the existing similarity between the catenary and parabola curves to derive analytical expressions of the tether state. 
The comparative analysis between the baseline and the proposed new method demonstrates that we can generate obstacle-free trajectories for UAVs, UGVs, and tethers in a fraction of the time while increasing the feasibility of the computed solutions. Finally, the source code of the method is publicly available.

\end{abstract}

\section{Introduction and Related Work}
\label{sec:introduction}
In recent years, there has been a notable increase in the development and research of tethered UAVs, reflecting a growing interest in their diverse applications. One of the main motivations is to carry out long-term missions with aerial vehicles, as these present significant challenges due to the limitations of current battery solutions \cite{robotics12040117}. A UAV tethered to a UGV is an interesting configuration, as the UGV can power the UAV through the tether for longer times given the higher payload of the former.  
This introduces a paradigm in robotic collaboration, offering distinct advantages over traditional standalone systems by combining the strengths of each of the robotic agents \cite{MooreIROS2018}. 
When deploying a UGV tethered to a UAV in scenarios requiring increased situational awareness and extended operational endurance, the tethered configuration can become even more invaluable, not only providing the UAV with power to significantly extend its flight time \cite{6961531},  
but also with safe high-bandwidth communications \cite{850822,9202196}. 

However, the tethering mechanism introduces several challenges, particularly in modeling the hanging tether state \cite{XiaoSSRR2018}. Unlike standalone systems, where each vehicle operates independently, the tether requires intricate and permanent coordination between the UGV and the UAV. Understanding and managing the state of the tether becomes a critical aspect, which requires sophisticated algorithms and real-time processing capabilities \cite{9561062}. 

\begin{figure}
  \includegraphics[width=0.2\textwidth]{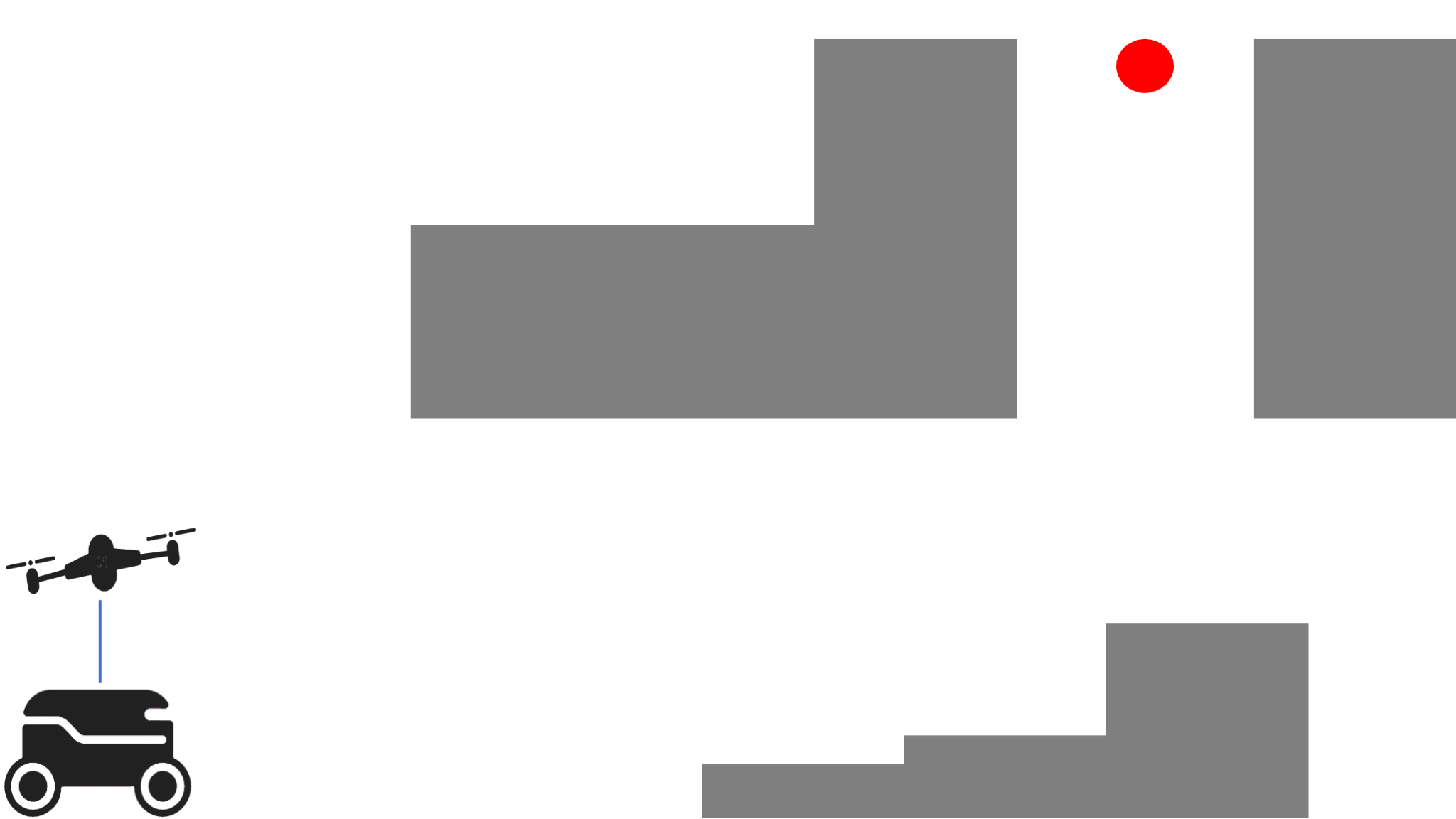}
  \hfill
  \includegraphics[width=0.2\textwidth]{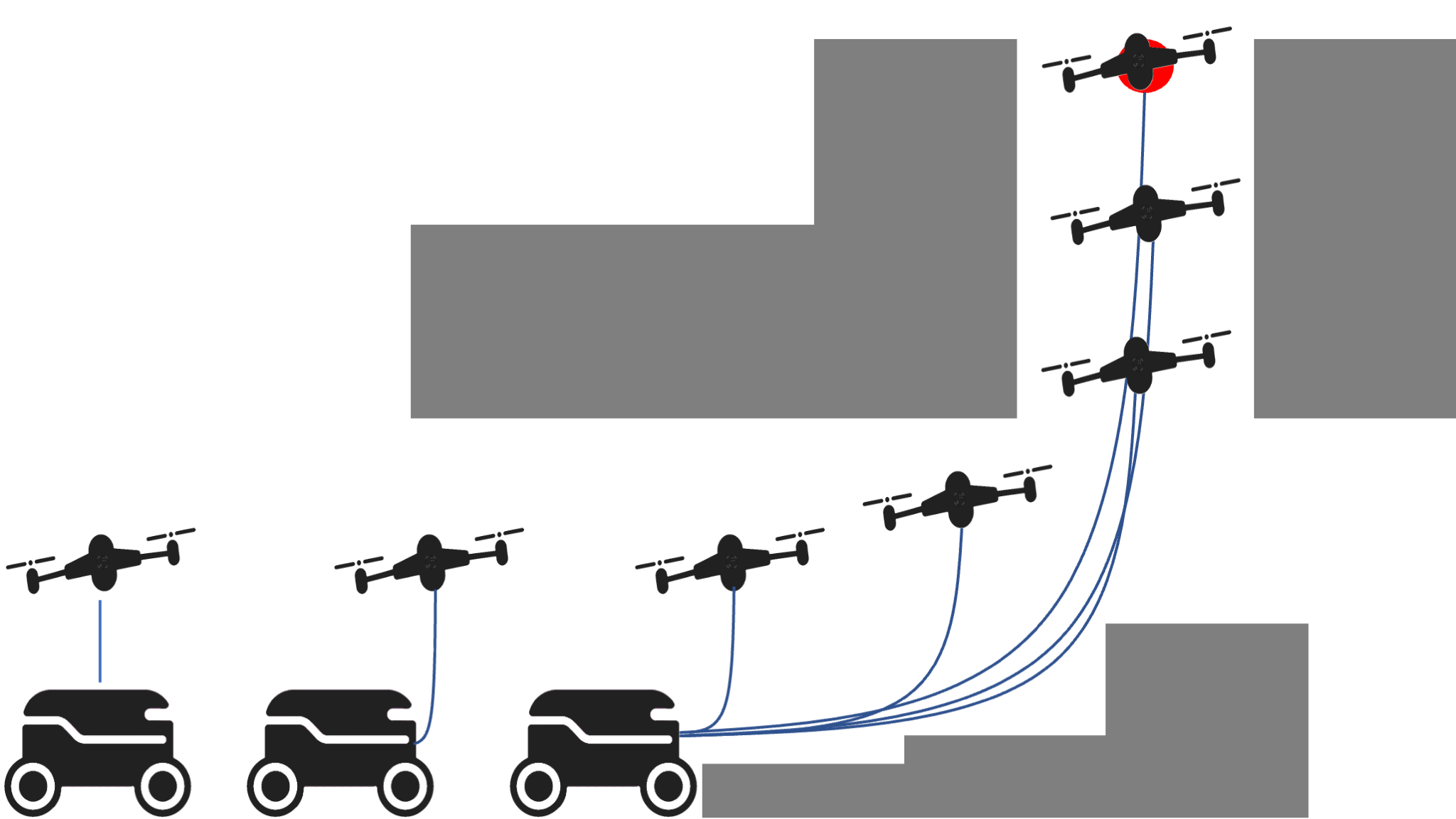}
  \caption{Simplified 2D sketch showing an example for motion planning of a tethered UAV-UGV with a hanging tether. (Left) Initial robots and tether configuration, and UAV goal (red circle). (Right) Sequence of robots positions and tether length to reach the given goal. Notice how the goal cannot be reached by means of a taut tether, a hanging tether must be considered in this case.}
  \label{fig:planning-setup}
\end{figure}

The state of the tether has traditionally been analyzed through parameterization, an approach that employs equations to represent its physical behavior, especially the catenary curve \cite{BOOKOFCURVES}. Numerous methodologies, with the aim of simplifying this process, approximate the tether as a straight line \cite{autonomousvisual}\cite{framworktether}\cite{uavfire}. This straight-line approximation is only suitable in scenarios where there is a direct line of sight between the tether endpoints, and thus it inherently restricts the exploratory range of the UAV.

In general, hanging-tether approaches allow UAVs to access a broader range of areas compared to straight tether setups; see Fig. \ref{fig:planning-setup} for an example. This concept has been explored by incorporating tether parameterization into localization or planning processes. For instance, Lima and Pereira \cite{9476778} use the catenary equation to determine the UAV's position.  
Similarly, in \cite{9364354}, the focus is on computing the state of a catenary tether to localize two UAVs attached at each end. This setup is specifically designed to suspend an object, providing a novel approach to object manipulation using UAVs while maintaining a constant tether length. Another interesting application of the catenary model is presented in \cite{LARANJEIRA2020107018} for underwater operations, where the catenary is used to monitor the status of a cable connected to an \emph{N}-number of ROVs (Remotely Operated Vehicles) performing exploration tasks, also with a constant tether length.

In \cite{8848946}, the parameterization of the tether is used in the localization and control stages to perform two autonomous motion primitives, reactive feedback-based position control and model-predictive feedforward velocity control, but is not used in the planning stage. An interesting approach is presented in \cite{drones7020073}, where a tied unmanned aerial vehicle (TUAV), named ``Oxpecke'', was designed for the inspection of stone-mine pillars. This system uses a sweeping (lawnmower) pattern path planning method intended to map and inspect an entire rectangular area, such as the surface of a pillar. However, the surface to inspect is simple (a rectangle), and the tether length is not directly included in the path planning.


A comprehensive approach to incorporate a tether in the planning stage is presented in \cite{battocletti2024entanglementdefinitionstetheredrobots}, where the authors define the challenges associated with tether entanglement. Specifically, this work addresses the constraints imposed by the tether on the motion of the robot, particularly the limitations arising from the finite length of the tether. Various constraints are integrated into the planning stage to account for these challenges. However, the proposed method is limited and mainly focused on ground applications, 
thus limiting its applicability to UAVs. Additionally, the approach allows for tether contact with the ground, as long as it does not result in entanglement.

On the other hand, \cite{capitán2024efficientstrategypathplanning} focuses on the development of a path planning strategy for marsupial robotic systems composed of a UGV tethered to a UAV. The article introduces a sequential planning strategy called MASPA (Marsupial Sequential Path-Planning Approach), which allows calculating collision-free 3D trajectories for the tethered UAV-UGV system in complex scenarios, for which the UGV advances to a point where the UAV executes the take-off and then advances to a desired point. This method considers both the geometric limitations imposed by obstacles and the cable and the properties of the joint motion of both robots. A novel algorithm, the PVA (Polygonal Visibility Algorithm), is also presented to identify feasible take-off points and solve visibility problems for the UAV in a three-dimensional space. Despite the novelty of the approach, it is not able to consider coordinated planning of the UGV and the UAV at the same time.

In \cite{smartinezr2023}, the catenary approximation is used to parameterize the state of the tether and plan a collision-free trajectory, in which the UAV must achieve objectives using a hanging tether. However, using the catenary equation, the planning process becomes a time-consuming task, allowing only offline computations. 

This paper focuses on reducing the complexity associated with the calculation of the variable length hanging tether. 
The main contributions are listed below.

\begin{itemize}
    \item A new method for efficient computation of a collision-free catenary curve based on the parabola approximation. This paper proposes using the parabola curve to model the hanging tether curve, detailing the full pipeline, including the computation of the final catenary model. This method reduces the execution time of the path planner to great extent, since it avoids the computational complexity associated with the calculation of the catenary model for tether collision detection. This model also increases the feasibility of the trajectory planner approach, reaching an averaged 98\% of feasibility in the validation scenarios. 

    \item A direct parameterization of the tether in the trajectory state definition, which includes the parameters of the curve (parabola or catenary) in the system state vector. This allows a more accurate evaluation of geometric constraints (such as distance to obstacles) and reduces the optimization time \rev{by more than an order of magnitude} compared to previous methods, achieving safer and smoother trajectories for the UAV-UGV system. \rev{Such improvement opens the door to apply the proposed method to real-time local re-planning.}
\end{itemize}



The paper is structured as follows. In Section \ref{sec:overview}, we show the general problem to be solved, whereas Section \ref{sec:approach} formalizes the solutions proposed. Section \ref{sec:path_planning} details the implementation of the solution within the planning stage. In Section \ref{sec:optimization_process}, we describe how curve parameterization is utilized to enhance the optimization process for trajectory computation. The experimental results are discussed in Section \ref{sec:experiments}. Finally, the paper is concluded in Section \ref{sec:conclusions}.

\section{Path and trajectory planning of tethered UAV-UGV in 3D overview}
\label{sec:overview}
The trajectory planning of a tethered UAV-UGV system can be summarized as the process of computing the sequence of obstacle-free UAV-tether-UGV positions, velocities, and accelerations that allow the UAV to reach a destination, given the starting configuration of the UAV-tether-UGV system (see Fig. \ref{fig:planning-setup} as an example). For the sake of generality, this paper focuses on a novel planning approach in which the tether has a variable and controllable length and might be hanging depending on the scenario. We can see in the example of Fig. \ref{fig:planning-setup} that the goal cannot be reached if only taut tether configurations are considered.

Using hanging tether configurations increases the chances of finding a suitable sequence of robot actions to the goal, but it also implies significant computational effort to model the tether shape. The catenary curve accurately models the shape of the tether and is the most common representation. Although fitting the catenary ${\cal C}$ can be easily implemented, the planner must solve this problem hundreds or thousands of times to check possible collisions of the tether with the environment in every UAV-UGV configuration, which becomes a time-consuming process. 

Our novel approach addresses the problem of efficiently computing the existence of a free-collision catenary ${\cal C}_{AB}$ that connects the position of the ground vehicle, $A$, and the position of the aerial vehicle, $B$. Adopting a common terminology, we call $A$ and $B$ the \emph{suspension points}, and the horizontal and vertical distance between $A$ and $B$ are called \emph{span} and \emph{sag}, respectively.

This paper proposes an approach that builds on top of the motion planning method for tethered UGV-UAV system presented in \cite{smartinezr2023}, where the computation of the catenary ${\cal C}$ is required as a representation of the tether's shape at two levels. First, in the path planning stage, they focus on a collision-free path for all agents. Second, in an optimization stage, where the objective is to optimize the previous path to obtain a trajectory, considering several constraints related to the UGV-UAV robotic configuration. 

\section{Efficient computation of a collision-free catenary curve between two points}
\label{sec:approach}





In this section, we present a new method for efficient computation of a collision-free catenary curve based on the parabola approximation. As was previously commented, the main idea is that the method enables the planner to calculate trajectories faster and more efficiently, since it avoids the computational complexity associated with the calculation of the catenary model for tether collision detection.

Considering the catenary as a planar curve, our input consists of a set of 2D obstacles $\cal{O}$ defined by polygons in a plane denoted $\pi$. This plane $\pi$ is perpendicular to the ground and passes through points $A$ and $B$. 


We define a curve $C$ as collision-free in $\pi$ if it does not intersect with any obstacle in $\pi$.  Therefore, our focus is directed towards the following two-dimensional problem:

\vspace{.25cm}
\textsc{Catenary Decision Problem (CDP):} \emph{Given a set $\cal O$ of polygons in $\pi$, determine whether there exists a collision-free catenary 
${\cal C}_{AB}$ in $\pi$ linking $A$ and $B$ with length at most $L$.} 



In \cite{smartinezr2023} a numerical approach was considered to determine the existence of a collision-free catenary suspended between points $A$ and $B$. The approach involves examining catenaries with increasing length, starting from $l = d(A,B)$ to $l=L$, with the addition of a specified increment $\Delta l$ until a collision-free trajectory is identified (success) or the maximum length is reached (failure). Furthermore, in each iteration, the catenary shape is sampled to verify for potential collisions.
Both the calculation of the catenary curve and the collision test are time-consuming. This involves solving transcendental equations, which is computationally demanding \cite{behroozi2014fresh}.
In addition, low values on $\Delta l$ could waste efforts by repeatedly calling the collision checker with very similar curves. In contrast, high values $\Delta l$ could cause the method to overlook some catenary lengths that might be collision-free.

\begin{figure}[t!]
  \centering
   \includegraphics[scale=0.35]{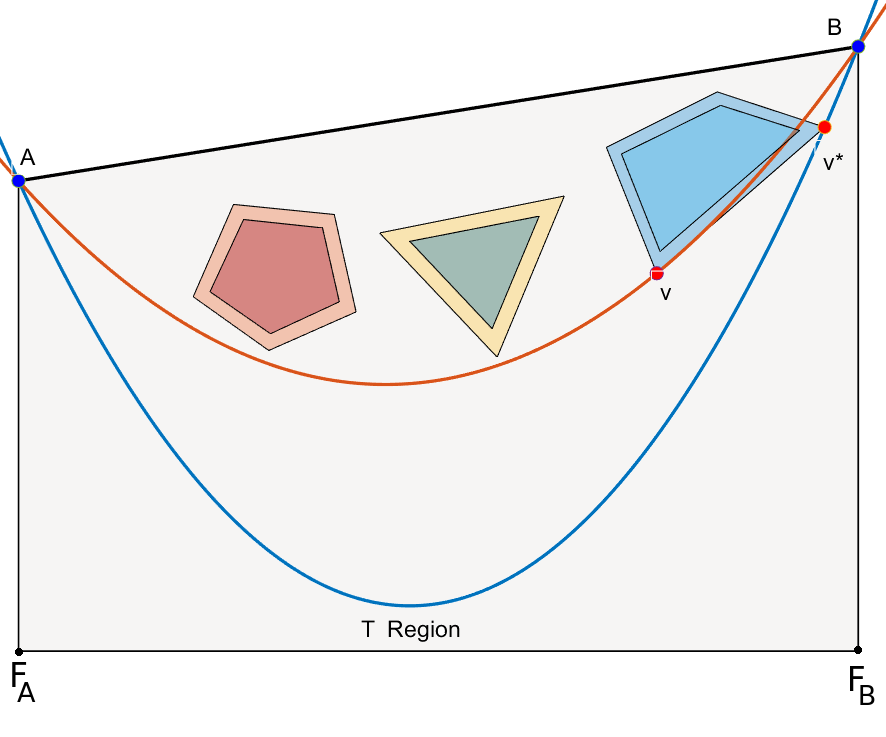}  
\caption{Polygons inside the $T$ region, which is delimited by the black trapezoid $\overline{ABF_AF_BA}$. $A$ and $B$ are the suspension points; $F_A$ and $F_B$ are their projections on the floor. The blue polygon is \emph{weakly inscribed}  in the parabolic region $R(v^*)$ and $\overline{AB}$, delimited by the blue parabola.}  
\label{fig:PDP}
\end{figure}

To overcome these issues, we have devised an alternative approach using a simpler curve.  In certain applications, such as the design of transmission overhead lines \cite{hatibovic2018algorithm}, the catenary has been replaced by a parabolic curve, as the parabola is a good approximation of a catenary if the sag is small \cite{hatibovic2020comparison}. However, in our scenario, the sag is relatively large, and we have to be careful with the parabolic approximation.
Moreover, our approach requires examining the reverse approximation: Given a collision-free parabola, we would like to compute an approximated catenary.

In the remainder of the section, we describe our proposed procedure for obtaining collision-free catenaries between two points. Initially, we address the decision problem using parabolas instead of catenaries. Subsequently, upon discovering a collision-free parabolic curve, we introduce a numerical method to approximate the parabola with a catenary. Finally, we ensure that the obtained catenary 
is collision-free by strategically expanding the obstacles.

Let $T$ be the open trapezoidal region in $\pi$ below the segment $\overline{AB}$, bounded by the ground and the two vertical lines passing through $A$ and $B$, as illustrated in Fig. \ref{fig:PDP}.

A parabolic curve between the points $A$ and $B$ divides the plane $\pi$ into two regions, one convex and one non-convex. Given a set of obstacles in the plane $\pi$, a parabolic curve is collision-free when all obstacles are contained in one of the two regions.

Let $P_{AB}(v)$ be the parabola defined by the points $A$, $B$, and $v$, where $v$ is a vertex of an obstacle in $T$. Denote by $\ell(v)$ the length of the parabola $P_{AB}(v)$ and
$R(v)$ the convex region defined by the segment $\overline{AB}$ and the parabola $P_{AB}(v)$.
A polygon is \emph{weakly inscribed} in the region $R(v)$ if it is contained in $R(v)$
and at least one vertex of the polygon lies on the parabola $P_{AB}(v)$. 
Therefore, we are interested in solving the following auxiliary problem:

\vspace{.25cm}
\textsc{Parabola Decision Problem (PDP):} 
\emph{Given a set $\cal O$ of polygons 
contained in $T$, decide whether there exists a collision-free parabola $P_{AB}(v)$ with a length at most $L$ passing through a vertex $v$ of a polygon in $\cal O$.}


\subsection{An iterative approach for solving PDP}
\label{solvingPDP}
Given a polygon $O=\langle v_1,\cdots,v_n\rangle$ contained in  $T$, we say that a parabolic arc hanging from points $A$ and $B$, $P_{AB}$, intersects with the polygon when it intersects a face of the polygon, i.e., there are points on a face that lie both above and below the curve. We denote  $P_{AB} \cap \mathcal{O}$  the subset of polygons of $\mathcal{O}$ that intersect the parabolic arc $P_{AB}$. We call this set $\mathcal{O'}$ and $CH(\mathcal{O'})$ its convex hull. In what follows, we will denote the length of $P_{AB}$ as $|P_{AB}|$ when we do not expressly know the third point $v$ through which the parabola passes.

\begin{lemma} \label{cor: longest_parabola}  
  Let $O=\langle v_1,\cdots,v_n\rangle$ be a convex polygon contained in the open region $T$ and let $v^*$ be the vertex of $O$ such that $\ell(v^*)=\max_{v_i \in O}\ell(v_i).$ Then the polygon $O$ is weakly inscribed in $R(v^*)$. 
  Furthermore, for any vertex $v_i \in O$ such that $\ell(v_i)<\ell(v^*)$, $O$ is not weakly inscribed in $R(v_i)$.
\end{lemma}	

\begin{proof}
Since two different parabolas intersect at most in two points, for all vertices $v_i\in O$ with $\ell(v_i) \leq \ell(v^*)$, it follows that $R(v_i)\subseteq R(v^*)$ and
the polygon $O$ is weakly inscribed in the region $R(v^*)$. On the other hand, for all $v_i$ so that $\ell(v_i) < \ell(v^*)$, $R(v^*) \nsubseteq R(v_i)$. Thus, $v^* \notin R(v_i)$ and $O$ is not weakly inscribed in $R(v_i)$. Fig. \ref{fig:PDP} illustrates the proof.
\end{proof}

As a direct consequence of Lemma \ref{cor: longest_parabola}, we state the following remark:
\begin{remark}
\label{rmk:longest_parabola} 
Let $O$ be a polygon
contained in $T$ that intersects a parabolic arc $P_{AB}$. 
Then the parabolic arc $P'_{AB}$ of minimum length that does not intersect with $O$ such that $|P'_{AB}|>|P_{AB}|$ in the span $(A, B)$ is the parabola of maximum length passing through a vertex of $CH(O)$.
\end{remark}

Algorithm \ref{alg:decision_problem} illustrates an overview of the approach to solving the parabola decision problem. We use the result presented in Remark \ref{rmk:longest_parabola} to progressively increase the length of the parabola until a specified stopping condition is met. To initiate the process, the algorithm computes the line segment $\overline{\rm AB}$ in line 1, representing the parabolic arc of minimum length from $A$ to $B$. The 
current parabola is stored in $P_{AB}$ and is subsequently updated throughout the algorithm. 
The fundamental invariant upheld at each step ensures that any parabola with a length less than 
$|P_{AB}|$ is not valid.
The algorithm enters a loop from lines [2-11], setting stop conditions for PDP in lines [2-8].

If the current parabola touches the ground or if its length reaches the maximum value, then there is no collision-free parabola hanging from $A$ to $B$, line [2-4]. Then we compute the set of obstacles $\mathcal{O'}$ intersecting with $P_{AB}$, line 5. If this set is empty, $P_{AB}$ returns as a valid parabola. Otherwise, we compute the minimum length parabola $P'_{AB}$ that does not intersect with $\mathcal{O'}$ such that $|P'_{AB}|>|P_{AB}|$. Notice that any parabola of length less than $P'_{AB}$ is not collision-free, so the established invariant is preserved. From Remark \ref{rmk:longest_parabola}, $P'_{AB}$ is the parabola of maximum length that passes through the convex hull of $\mathcal{O'}$; hence, we update $P_{AB}$ with this parabola, lines [9-10]. At this point, a full cycle of the algorithm has been completed, so we return to line 2. The iterative process guarantees termination, as the algorithm systematically increases the length of the current parabola.

\begin{algorithm}[t!]
\caption{Given the suspension points $A$ and $B$, and a set of polygons $\mathcal{O}$ inside $T$, 
returns, if it exists, a collision-free parabolic arc $P_{AB}$ with maximum length $L$. 
$P_{AB}$ cannot touch the ground.
}
\label{alg:decision_problem}
\begin{algorithmic}[1]
\Require $A, B, \mathcal{O}, L$.
\Ensure Collision-free parabolic arc $P_{AB}$.

 \State $P_{AB} \gets \overline{\rm AB}$
\If{${P}_{AB}$ touches the ground \textbf{or} $|P_{AB}|> L$} 
    \State \Return None 
\EndIf

\State $\mathcal{O}' \gets P_{AB} \cap \mathcal{O}$
\If{$\mathcal{O}' = \emptyset$}
    \State \Return $P_{AB}$
\EndIf

\State $v^* \gets \text{argmax}_{v_i \in CH(\mathcal{O}')}\ell(v_i)$
\State $P_{AB} \gets P_{AB}(v^*)$
\State GoTo (line 2)
\end{algorithmic}  
\end{algorithm}



\subsection{Computing the collision-free catenary from a parabola}
\label{sec:approxlength}

Algorithm \ref{alg:decision_problem} yields a parabolic curve that solves the PDP. In this section, we introduce three methods for approximating the parabolic curve with a catenary, representing the shape of a hanging tether. Following an experimental comparison, we choose one of these methods.

\subsubsection{Method 1. Approximation by length}
\label{sec:bylength}
In our first method, we just impose the catenary to be of the same length as the obtained parabola. Let $\lambda$ be the length of the parabola that connects the suspension points $A=(0, h_1)$ and $B=(S, h_2)$, and let $(x_{min}, y_{min})$ be the coordinates of the lowest point of the catenary with length $\lambda$ passing through $A$ and $B$. Note that $h_1, h_2, x_{min},  y_{min} > 0$ and $\lambda > d(A,B)=\sqrt{S^2 +(h_2-h_1)^2}$.
Taking the general equation of the catenary \cite{hatibovic2020comparison}: 
\begin{equation}
y_{cat}(x)=2c\sinh\left(\frac{x-x_{min}}{2c}\right)+y_{min}, c>0,
\end{equation}
the solution of the following system of non-linear equations provides the expected catenary: 
\begin{equation}
\left.
    \begin{array}{lll}
    2c \sinh^2 {\left(\frac{-x_{min}}{2c}\right)} + y_{min} & = & h_2 \\
    2c \sinh^2 {\left(\frac{S-x_{min}}{2c}\right)} + y_{min} & = & h_1 \\
    c\left[ \sinh{\left( \frac{S-x_{min}}{c}\right)}-   \sinh{\left( \frac{-x_{min}}{c}\right)} \right] & = & \lambda
    \end{array}
    \right\}
\end{equation}

\medspace



\subsubsection{Method 2. Fitting the catenary with a parabola}
\label{sec:approxfitting}
An alternative method is to find a catenary that fits the parabola, specifically the catenary that minimizes the maximum vertical distance with the parabola. To design an efficient fitting algorithm, some mathematical properties are used.

From the convexity of the catenary, the following Lemma can be stated:
\begin{lemma}
\label{le:big_area_small_length}
    Let $\cat_1$, $\cat_2$ be two catenaries hanging from the same suspension points $A=(x_1, h_1)$ and $B=(x_2, h_2)$ such that $\forall x\in (x_1, x_2)$ $\cat_1(x)> \cat_2(x)$. Then the area under $\cat_1$ in $[x_1, x_2]$ is greater than the area under $\cat_2$ in the same interval, but its length is lower.
\end{lemma}

We use an additional result extracted from \cite{parker2010property}:
\begin{theorem}
\label{theo:cool_cat_property}
    Given a catenary curve $\cat$ and any horizontal interval $(x_1, x_2)$, the ratio defined by the area under $\cat$ divided by the length of the curve in that interval is independent of the value of $x_1$ and $x_2$.
\end{theorem}

Now we are ready to prove the following statement:
\begin{theorem}\label{teo:only_two_in_common}
Two different catenaries hanging from the same suspension points cannot have more than two points in common.
\end{theorem}
\begin{proof}
    (By contradiction). Let $\cat_1$, $\cat_2$ be two catenaries that hang from $A=(x_1, h_1)$ to $B=(x_2, h_2)$ such that they intersect at at least three points within the interval $[x_1, x_2]$. Let $a=(x_a, h_a)$, $b=(x_b, h_b)$, $c=(x_c, h_c)$ with $x_a<x_b<x_c$ be the intersection points,
    and let $\cat_i^a$ and $\cat_i^c$ indicate the arcs of the catenary $\cat_i$ contained in the interval $[x_a, x_b]$ and $[x_b, x_c]$, respectively. Without loss of generality, we assume $\forall x\in (x_a, x_b)$ $\cat_1^a(x)> \cat_2^a(x)$; hence $\forall x\in (x_b, x_c)$ $\cat_1^c(x)< \cat_2^c(x)$. Let $r_1$ ($r_2$) be the area under $\cat_1$ ($\cat_2$) in $[x_a, x_b]$ divided by the length of $\cat_1^a$ ($\cat_2^a$). Let $A_1, A_2, C_1, C_2$ be the area under $\cat_1^a,\cat_2^a, \cat_1^c,\cat_2^c$ respectively, and $L_1, L_2, L_3, L_4$ their lengths. From Theorem \ref{teo:only_two_in_common}, $r_1=\frac{A_1}{L_1}=\frac{C_1}{L_3}$ and $r_2=\frac{A_2}{L_2}=\frac{C_2}{L_4}$. In addition, from Lemma \ref{le:big_area_small_length}, we know that $A_1=A_2+\epsilon_1$ ($\epsilon_1>0$), and $L_1<L_2$; then 
    \begin{equation}
    r_1=\frac{A_2+\epsilon_1}{L_1}>\frac{A_2+\epsilon_1}{L_2}>\frac{A_2}{L_2}=r_2  
    \end{equation}
    \noindent On the other hand, from Lemma \ref{le:big_area_small_length} we know $C_2=C_1+\epsilon_2$ ($\epsilon_2>0$), and $L_3>L_4$; then 
    \begin{equation}
    r_2=\frac{C_1+\epsilon_2}{L_4}>\frac{C_1+\epsilon_2}{L_3}>\frac{C_1}{L_3}=r_1
    \end{equation}
    \noindent From the above we get 
    a contradiction and the result follows. 
\end{proof}



As a direct consequence of Theorem 
\ref{teo:only_two_in_common}
we have: 
\begin{corollary}\label{cor:monotony}
Let ${C}_{AB}^{i}(x)$ be the catenary function of length $l_i$ that connects points $A$ and $B$. If $l_i \geq l_j$, $j\neq i$, then ${C}_{AB}^{j}(x) \geq {C}_{AB}^{i}(x)$, for all $x\in [0,S]$.
\end{corollary}


The monotonicity property of Corollary \ref{cor:monotony} allows us to implement a bisection method within the allowable length interval of the tether, namely $[d(A, B), L]$, to identify an appropriate catenary setting.
As demonstrated in our experiments, this method proves to be efficient, delivering accurate approximations with minimal iterations. Algorithm \ref{alg:catenaryaproximation} depicts the pseudocode of the solution.

\begin{algorithm}[t!]
\caption{Algorithm to compute a fitting catenary.}
\label{alg:catenaryaproximation}
\begin{algorithmic}[1]
\Require \emph{point} $A$, \emph{point} $B$, \emph{parabola} ${P_{AB}}$, \emph{maximum length} $L$, \emph{approximation error} $\varepsilon$
\Ensure $\mathcal{C}_{AB}:$ a fitting catenary hanging from $A$ and $B$.
\State $L_0 \gets |\overline{AB}|$
\State $L_1 \gets L$
\While{$True$}
\State $L_m \gets (L_0+L_1)/2$
\State $\mathcal{C}_{AB} \gets get\_{cat}atenary(A,B,L_m)$
\If{$L_1-L_0 \leq \varepsilon$}
    \Return $\mathcal{C}_{AB}$
\EndIf
\State $x_{max} \gets get\_max\_distance\_axis({P_{AB}},\mathcal{C}_{AB}, A, B)$
\State $d \gets {P_{AB}}(x_{max}) - \mathcal{C}_{AB}(x_{max})$
\If{$d < 0$}
    \State $L_0 \gets L_m$
\Else{}
    \State $L_1 \gets L_m$
\EndIf

\EndWhile
\end{algorithmic}
\end{algorithm}

In Algorithm \ref{alg:catenaryaproximation}, the suspension points $A=(0, h_1)$ and $B=(S, h_2)$ and the maximum length of the expected catenary, $L$, are taken as input. In addition, an approximation error $\varepsilon$ is received to determine the accuracy of the output. The metric used in the bisection is the difference in terms of lengths between two iterations. First, the minimum and maximum lengths of the expected catenary are defined as $L_0$ and $L_1$, respectively, lines 1 and 2. These values are obtained from the parameters $A$, $B$, and $L$ and define a feasible interval of catenary lengths. Later, the algorithm performs a loop to estimate the optimal catenary length in the range $[L_0, L_1]$, lines 3-15. In each iteration, the mean length between $L_0$ and $L_1$, $L_m$, is calculated, obtaining the catenary of length $L_m$, $\mathcal{C}_{AB}$, lines 4 and 5. In the next step, the value: 
\begin{equation}
 x_{max} = max_{x \in [0, S]}|{P_{AB}}(x) - \mathcal{C}_{AB}(x)|   
\end{equation}
\noindent is calculated using the function $get\_max\_distance\_axis$ on line 8, and the maximum vertical difference $d$ between ${P_{AB}}$ and $\mathcal{C}$ is obtained on line 9. The value $x_{max}$ can be obtained by solving a numerical problem or can be approximated by sampling the interval $[0, S]$. Using Corollary \ref{cor:monotony}, the sign of the difference $d$ is the condition to discard an entire subinterval, $[L_0,L_m]$ or $[L_m,L_1]$.
The stop criterion is outlined at line 6, and the algorithm returns the last calculated catenary $\mathcal{C}_{AB}$. 

\subsubsection{Method 3. Fitting the catenary by sampling the parabola points}
\label{sec:bypoints}

This method consists of fitting the catenary to the parabola through a nonlinear optimization process to minimize the distance to $n$ points sampled on the catenary and the parabola. For this, the $n$ points of the parabola are considered in a plane X-Y. The nonlinear optimizer creates $n$ constraints that must minimize the Y-distance between both curves, parabola and catenary. This translates into the following optimization problem.

\begin{equation}
     cat(x)^*=\arg \min_{cat} \sum_{i=1}^n ||par(x_i) - cat(x_i)||^2
\end{equation}


\noindent where $cat(x)=a\ cosh( (x-x_0)/a) + y_0$ represents the catenary equation in which the parameters $a$, $x_0$ and $y_0$ must be estimated, and $par(x)$ is the parabola equation to fit.

\subsubsection{Benchmarking}

Now, the methods for computing an approximation catenary are compared each other. The methods of Sections \ref{sec:bylength}, \ref{sec:approxfitting}, and \ref{sec:bypoints} are denoted by \emph{ByLength}, \emph{ByFitting} and \emph{BySampling}, respectively.
A set of random experiments was performed to compare the methods in terms of the maximum vertical distance between the reference and the approximated curves. In the experiments, the approximation error $\varepsilon$ selected for the \emph{ByFitting}, as appears in Algorithm \ref{alg:catenaryaproximation}, was $10^{-2}$, and the number $n$ of points in \emph{BySampling} was 5.

A set of 100 experimental scenarios was generated. For each scenario, the suspension point $A$ is set to $(0,1)$, while both the suspension point $B$ $(B_x,B_y)$ and the parabola's longitude are generated randomly. In accordance with the intended application, the parabolic curve is generated in such a way that the point $B$ is positioned higher than the point $A$. The maximum allowed curve length $L$ is set to $30$ meters.

Two metrics within the interval $[0,B_x]$ are taken into account to evaluate the performance of the approximation methods:
\begin{itemize}
        \item [1.] $\epsilon_{mean}$: The mean vertical distance in the interval $[0,B_x]$ between both curves, catenary and parabola.  
        \item [2.] $\epsilon_{mean}/L$:  Is the mean vertical distance over the length of the curve in the interval $[0,B_x]$.
\end{itemize}



\noindent Table \ref{tab:comparison} shows the mean and standard deviation of the results of all experimental scenarios. It can be seen that \emph{BySampling} presents the best performance among the methods in terms of vertical distance, which means that the catenary and parabola curves are similar.



\begin{table}[t!]
\caption{
Comparison between the approximation methods: by Length, by Fitting and by Points
For each metric, the mean and standard deviation are shown.}

\begin{center}
  \small
  \scalebox{0.9}{
\begin{tabular}{|c|c|c|} 
 \hline
 \textbf{Method} & $\epsilon_{mean}~[m]$ & $\epsilon_{mean} / L$  \\ [0.8ex]
 \hline
 \emph{ByLength} & $0.3187\pm 0.3299$  & $0.0200\pm 0.0214$ \\
 \hline
 \emph{ByFitting} & $0.3127 \pm 0.3108$ & $0.0194 \pm 0.0200$ 
 \\
 \hline
  \emph{BySampling} & $0.2841 \pm 0.3040$ & $0.0173 \pm 0.0178$ 
 \\
 \hline
\end{tabular}
}
\end{center}
\label{tab:comparison}
\end{table}



\subsection{Ensuring a collision-free catenary}

Although the proposed catenary approximations of the parabola are accurate enough, we need to be sure that the computed catenaries are feasible, that is, collision-free and with admissible length.

Ensuring a feasible tether length after the catenary approximation can be easily implemented by constraining the parabola length during optimization. We can set the minimal tether length a $5\%$ longer and the maximum length a $5\%$ shorter, ensuring that the computed catenary will still hold the length constraint even with the distortions applied by the approximation.

However, ensuring a collision-free catenary is more complex. The most usual approach to solve this problem consists in inflating the obstacles, as in Algorithm \ref{alg:decision_problem}. The computed parabola avoids obstacles and passes through a vertex of an obstacle inside the trapezoidal region $T$. To guarantee that the calculated approximate catenary remains collision-free, the obstacles should be initially expanded as follows. For each polygon, we consider the \emph{expanded} polygon, which is defined by lines parallel to the faces of the polygon and located at a distance of $\tau$ (see Fig. \ref{fig:PDP}). The tangent collision-free parabola is built in the Algorithm \ref{alg:decision_problem} taking as input the expanded polygons. In practice, it is crucial to determine the tolerance $\tau$ for both ground and aerial obstacles. Large values of $\tau$ can give negative responses to the problem PDP, while small values of $\tau$ can lead to non-feasible approximation catenaries.  The election of $\tau$ strongly depends on the approximation method used. The more accurate the method, the better the selection of $\tau$. According to Table \ref{tab:comparison}, the most accurate approximation is \emph{BySampling} with a relative error of $0.0173(\sigma:0.0178)$ per meter. We can use such error to estimate the value of $\tau$, setting it to the average approximation error plus its standard deviation is a good trade-off, particularly $\tau=0.035*L$, where $L$ is the longitude of the tether. Note that the mean vertical distance between the parabola and the catenary never exceeded this value in the experiments. This means that both the ground and the obstacles in $\mathcal{O}_{\tau}$ should differ only in $0.6$ meters in the vertical axis from their original shape. Therefore, using \emph{BySampling} and $\tau=0.6~m$, we can ensure that Algorithm \ref{alg:catenaryaproximation} solves the CDP problem and computes a collision-free catenary for the marsupial robotic system.

While the previous method is safe and effective, obstacle inflation suffers from over-constraining the planning problem, eliminating possible feasible solutions under the assumption of the worst-case scenario. Instead, we propose not using obstacle inflation, but reevaluating the solution (now with catenary) to check if it is still collision-free after fitting the parabola. In case the solution is not feasible, we can slightly adjust the length of the tether so that it is collision-free again.

\section{Efficient Path Planning Approach for a tethered UGV-UAV system}
\label{sec:path_planning}

In the previous section, we defined a procedure that allows us to efficiently identify the status of a collision-free catenary using a decision problem. In this section, we use this procedure \rev{to enhance the path planning method for a tethered UAV-UGV robotic configuration proposed in \cite{smartinezr2023}. } We summarize the method here for the sake of completeness.


The goal of the path planning algorithm is to devise a safe path for the marsupial system that connects a starting position of the whole system to a goal configuration in which the UAV system has a goal position, while the rest of the system has an arbitrary, but feasible and collision-free, configuration. To this end, we use the RRT* algorithm \cite{karaman_rrt_star}.

To reduce the complexity of the approach, we only consider as decision variables the position of both platforms, omitting the variables related to the state of the tether. 
These tether parameters are then obtained by solving the corresponding PDP or CDP associated with the positions of both platforms.

The path planner in this Section does not consider any kinematic or dynamic constraints to obtain the path, leaving them for the optimization stage of the algorithm described in Section \ref{sec:optimization_process}.


The main steps of the RRT* algorithm are as follows. 
It creates a tree starting from the initial configuration of the marsupial system which will be expanded in a loop for a given number of iterations. In the loop, it generates a new collision-free sample in the \textit{Sampling} step, generating a random collision-free node ($x_{random}$) containing the positions of the UAV and the UGV. Then, RRT* tries to extend the tree towards the random sample, obtaining a new candidate node to be added to the tree ($x_{new}$). The new node should be \textit{validated} by solving the Decision Problem  of section \ref{sec:approach} connecting the new UAV and UGV poses. If extended, RRT* optimizes the graph by searching for the best parent node and rewiring the tree in the neighborhood of $x_{new}$. These steps lead to an asymptotically optimal solution for the path \cite{karaman_rrt_star}. We explain the main particularities of our implementation in each step below.

\subsection{Sampling}

In this step, we search for valid UGV and UAV positions. To this end, the position of the UGV is sampled at the traversable points of the 3D map with a minimum clearance \cite{driving_pc}. Similarly, the position of the UAV is sampled in the obstacle-free space of the environment, which is composed of points having a minimum clearance. We use Euclidean Distance Fields (EDF) to speed up the sampling process, storing it in a preprocessed grid, which contains the distance from each grid point to its closest obstacle in the environment \cite{edf_survey}. In this way, we can check if a point is collision-free just by checking its EDF.




\subsection{Steer}

\label{sec:steering}

This step tries to generate a new valid node ($x_{new}$) by extending the tree from the nearest node ($x_{near}$) of the previously generated  configuration ($x_{random}$). We follow the steering function proposed in \cite{smartinezr2023}, which sequentially tries three different steering modes, keeping the new node of the first successful mode. Note that in contrast to the Sampling step, the new node should be validated before adding it to the tree. In each mode, the steering method considers movement in a subset of the agents:

\begin{enumerate}
    \item The first mode steers only the UAV position component, and the UGV position is fixed. 
    \item In the event that the first mode does not succeed, the second mode is executed. The second mode steers both the UGV and the UAV
    \item The last mode just steers the UGV and considers the UAV fixed.  
\end{enumerate}

\subsection{Validating the new node}

Whenever a candidate $x_{new}$ configuration is generated with one of the steering alternatives, we make sure that there exists a collision-free tether connecting the positions of each vehicle. To this end, a decision problem, either CDP or PDP, should be solved. 



In this paper, we propose to use a parabolic model of the tether to speed up the RRT* algorithm. Therefore, we solve a PDP using our iterative algorithm of Section \ref{solvingPDP}. Note that the proposed iterative algorithm uses a 2D description of the environment and obstacle clustering to accelerate the processing, while the RRT* is working in 3D. In the rest of the section, we detail our proposed procedure that conveniently projects and clusters the obstacles in the environment to a 2D plane, making use of the input EDF.

Once the positions of the UGV and UAV have been defined, we compute the 2D vertical plane that passes through both positions. Then, we uniformly sample the 3D EDF in the T region 
(see Fig. \ref{fig:PDP}). If the EDF at a sample point does not meet a minimum clearance, we assume that the point is an obstacle. Otherwise, it is considered as free space. We use the same resolution in this sampling procedure as the resolution of the provided EDF.

While we are generating the grid containing the obstacles in the 2D plane, we try to group them by connecting the  currently generated point with the previous obstacles in its already computed neighborhood, if any. This method is fast, can be performed just by analyzing the direct neighbors of each sample, but might divide concave obstacles into several clusters depending on their shapes. As an example, Figure \ref{fig:clustering} represents a projection that has been properly clustered into three main obstacles (red, green, and blue). However, the clustering algorithm fails to merge the remaining obstacles in pink, purple, and yellow, to name a few. Therefore, the computation of the PDP will increase slightly due to a larger number of clusters. In spite of it, results of Section \ref{sec:experiments} show that this increase is marginal and that PDP remains the fastest method to solve the problem. 

\rev{Once we have obtained the different clusters as point clouds, we get the convex hull of each cluster, resulting in the input polygons of Algorithm \ref{alg:decision_problem}.}


\begin{figure}[!t]
  \includegraphics[width=0.48\textwidth]{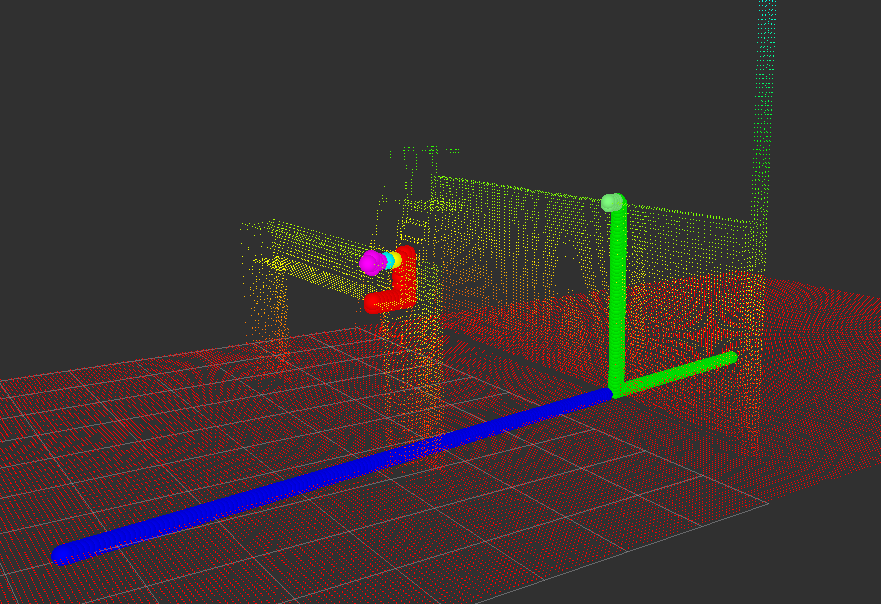}
  \caption{Example of a 2D projection, in which each cluster is represented with a set of big dots in a different color. The projection is obtained from a 3D Point Cloud, represented in fine-grained dots with colors from red to green depending on its $z$ coordinate. Figure obtained with RViz \cite{rviz}. }
  \label{fig:clustering}
\end{figure}

\section{Trajectory optimization}
\label{sec:optimization_process}
Once an initial feasible path is computed, we proceed to optimize the full robot system trajectory as a whole, considering dynamic constraints such as velocities or accelerations, and distance to obstacles. 

In this regard, this paper follows the methodology presented in \cite{smartinezr2023}. Unlike \cite{smartinezr2023}, instead of considering the length of the tether in the state vector, here we consider the set of parameters that actually define the tether curve. For both the catenary and parabola, the number of required parameters is three, so we increase the dimension of the state vector in two.  
However, this new state vector improves the convergence and decreases the computation to obtain a feasible solution. Including the curve parameters in the state vector (instead of just the length) enables the optimizer to directly sample the tether and properly compute the gradients with respect to the curve itself. Section \ref{sec:experiments} will show that this new state vector reduces the optimization time by at least one order of magnitude.


Thus, in this paper, the state of the system in each time step includes the position of the UGV $\mathbf{p}_g=(x_g,y_g,z_g)^{T}$, the position of the UAV $\mathbf{p}_a=(x_a,y_a,z_a)^{T}$, the parameters that define the curve of the tether $\mathbf{T}$, and the time variable $t$. For consistency, the UGV, UAV, and tether trajectories must use the same $t$ to coordinate the arrival of both platforms at the planned waypoint. 

We discretize the trajectories so that the states of our problem become the set:
\begin{equation}
\label{eq:traj_params}
\Omega = \{\mathbf{p}^i_g,\mathbf{p}^i_a,\mathbf{T}^i,\Delta t^i\}_{i=1,...,n}
\end{equation}

\noindent where $\Delta t^{i} = t^{i} - t^{i-1}$ is the time increase between states $i$ and $i-1$, allowing temporal aspects. When needed, we also discretize the resultant tether model into a set of $m$ positions $\mathbf{p}_{t}=(x_{t}, y_{t}, z_{t})$:

\begin{equation}
\label{eq:tether}
    P^i = \{\mathbf{p}^j_t\}_{j=0,...,m-1}
\end{equation}

Our problem consists of determining the values of the variables in $\Omega$ (\ref{eq:traj_params}) that optimize a weighted multi-objective function $f(\Omega)$:  

\begin{equation}
\label{eq:cost_function}
     \Omega^* = \arg \min_\Omega f(\Omega) = \arg \min_\Omega  \sum_{i,k} \gamma_k * || \delta_k^i(\Omega) ||^2
\end{equation}

\noindent where $\Omega^*$ denotes the optimized collision-free trajectory for the UAV, the UGV, and the tether from the starting to the goal configurations. $\gamma_k$ is the weight for each component $\delta_k(\Omega)$ (also known as residual) of the objective function.  
 Each component encodes a different constraint or optimization objective for our problem and should be evaluated in all the time steps $i$.
In addition, the optimization problem is addressed using nonlinear sparse optimization algorithms, with \emph{Ceres-Solver} \cite{ceres-solver} serving as the optimization back-end.



\subsection{Tether parameterization}
\label{sec:tether_params}

We propose two different curves to model the tether: the catenary and the parabola. Both are parameterized by three variables:
\begin{eqnarray}
    parabola(x) &=& px^2+qx+r \\
    catenary(x) &=& a\ cosh( (x-x_0)/a) + y_0 
\end{eqnarray}

\noindent Thus, the parameters of $\mathbf{T}$ in (\ref{eq:traj_params}) will be defined as \mbox{$\mathbf{T}=(a, x_0, y_0)^T$} for the catenary curve, and $\mathbf{T}=(p,q,r)^T$ for the parabola.


\subsection{Initialization}
\label{sec:mathappoach}


The variables in the optimization (positions and tether states) are initialized from the feasible path provided by the RRT* planner of Section \ref{sec:path_planning}. This initial path consists of a sequence of UAV and UGV positions with catenary lengths at each position. To initialize each $\mathbf{T}_i$ we need the parabola that best fits such a catenary.


The traditional approach \cite{9337759} to find the parabola parameters that best fit a catenary involves assigning the length of the catenary $L_C$ to the length of the parabola. Thus, the problem of computing the parameters of the parabola is reduced to computing the parameters of the curve passing through two suspension points, \emph{A} and \emph{B}, with a given length $L_C$. This can be obtained by solving the following nonlinear system of equations:


\begin{equation}
    \label{eq:parabola_2}
    \left.
        \begin{array}{cc}
            y_A = p{x_A^2} + qx_A + r & \\
            y_B = p{x_B^2} + qx_B + r & \\
            \vphantom{\int_{A}^{B}} L_C=\int_{A}^{B}\sqrt{1+(2px+q)^2}\,dx
        \end{array}
        \right\}
\end{equation}




\noindent Notice that this integral can be solved analytically, but its solution is a complex equation establishing highly nonlinear relations among the parabola parameters. This makes the solution sensitive to parameter initialization.

To avoid nonlinear calculation, we propose an approach consisting of equalizing the area under the parabola between the segment connecting \emph{A} and \emph{B} with the area under the catenary curve, $a_C$. Using the parabola equation $px^2+qx+r$, we define a system of equations to obtain the parabola parameters \emph{p}, \emph{q} and \emph{r} in 2D (as a plane) as follows:




\begin{equation}
    \label{eq:parabola_4}
    \left.
        \begin{array}{lll}
            y_A = p{x^2}_A + qx_A + r & \\
            y_B = p{x^2}_B + qx_B + r & \\
             a_C =  \int_{A}^{B} (px^2 + qx + r)  \,dx       
        \end{array}
        \right\}
\end{equation}

\noindent This is a relatively simple linear system of equations for $p,\ q,\ r$ and $y$, that can be solved efficiently. We will use this method to compute the best-fitting parabola.

The computed parabola parameters are used as part of the initial solution for the optimization process. Through this optimization, we obtain the optimized parabola parameters and then compute the corresponding catenary by means of the Algorithm \ref{alg:catenaryaproximation}. 

\subsection{Optimization process}
\label{sec:implementation}


The optimizer takes into account geometric constraints such as obstacle avoidance, trajectory smoothness, and equi-distance between states. In addition, it considers temporal constraints including time, velocity, and acceleration, applicable to both the UGV and the UAV platforms. 
The main set of constraints is inherited from the optimization method detailed in \cite{smartinezr2023}. In particular, we use the penalty functions related to the UAV, the UGV and the tether feasibility. However, we redefine the maximum tether length constraint and we include a new residual that ensures that the tether passes through the positions of both the UAV and UGV due to the new parametric formulation.




Next, we present the new constraints required to consider the direct tether parameterization presented in (\ref{eq:traj_params}) for parabola and catenary:





\subsubsection{Length constraint}
This constraint penalizes unfeasible tether lengths for the catenary and parabola. The tether cannot be shorter than the Euclidean distance between the UGV and the UAV $d^i_u=||\mathbf{p}^i_g-\mathbf{p}^i_a||$ and cannot exceed its maximum length $L_{max}$. Given the tether parameters $\mathbf{T}^i$ (parabola or catenary), we can analytically compute the length of the curve between the suspension points $l^i$. With this information, we define the following residual:
\begin{equation}
    \label{eq:eq_length}
    \delta^i_{up} =   e^{d^i_{u} - l^{i}} + e^{l^{i} - L_{max}} 
\end{equation}
\noindent Notice how the residual starts rising over zero when approaching the minimum and maximum thresholds, growing exponentially as we pass the thresholds.

\subsubsection{Parabola parameter constraint}
This is only applied when $\mathbf{T}^i$ is a parabola; it penalizes unfeasible solutions for $p^{i}$, $q^{i}$, and $r^{i}$. This constraint ensures that the parabola passes through the positions of the UGV and UAV, $\mathbf{p}^i_g$ and $\mathbf{p}^i_a$. The constraint projects such 3D positions to the 2D plane that contains the points and is perpendicular to the floor, obtaining the 2D suspension points A and B. The suspension points must comply with the following equations: 


\begin{equation}
  \label{eq:eq_parameters_parabola}
  \left.
    \begin{array}{lll}
        \delta^i_{pA} = p{x^2_A} + qx_A + r - y_A \\
        \delta^i_{pB} = p{x^2_B} + qx_B + r - y_B 
    \end{array}
    \right\}
\end{equation}

\subsubsection{Catenary parameter constraint}
This is only applied when $\mathbf{T}^i$ is a catenary; it penalizes unfeasible solutions for $a^i$, $x^i_0$, and $y^i_0$. This constraint ensures that the catenary passes through the positions of the UGV and UAV, $\mathbf{p}^i_g$ and $\mathbf{p}^i_a$. The constraint projects such 3D positions to the 2D plane that contains the points and is perpendicular to the floor, obtaining the 2D suspension points A and B. The suspension points must comply with the following equations: 


\begin{equation}
  \label{eq:eq_parameters_catenary}
  \left.
  \begin{array}{lll}
  \delta^i_{cA} &=& a\ cosh(\frac{x_A- x_0}{a}) + y_0 - y_A \\
  \delta^i_{cB} &=& a\ 
  cosh(\frac{x_B- x_0}{a}) + y_0 - y_B
  \end{array}
\right\}
\end{equation}



\section{Experimental Results}
\label{sec:experiments}

Here, we focus on the experimental validation of the proposed method. We will benchmark the solution against \cite{smartinezr2023}, which is, to the best of our knowledge, the only approach for planning the path and trajectory of a UGV and UAV linked by a 3D hanging tether. To this end, the proposed approach will be tested/validated in the same scenarios (S1, S2, S3, S4, S5) and the same initial and final configurations of the robotic system as in \cite{smartinezr2023}, which are publicly available\footnote{\url{https://github.com/robotics-upo/marsupial_optimizer/tree/noetic/experiments_execution_instructions}}. These scenarios are illustrated in Fig. \ref{fig:scenarios}. 


All the experiments have been run on the same computer, an eight-core AMD Ryzen 7 6800H CPU, 32 GB of RAM running Ubuntu 20.04 and ROS Noetic. The source code is available on an anonymous GitHub repository \footnote{\url{https://anonymous.4open.science/r/efficient_tether_parameterization-040F/}.}.

\begin{figure*}[t!]
  \centering
  \includegraphics[width=0.9\textwidth]{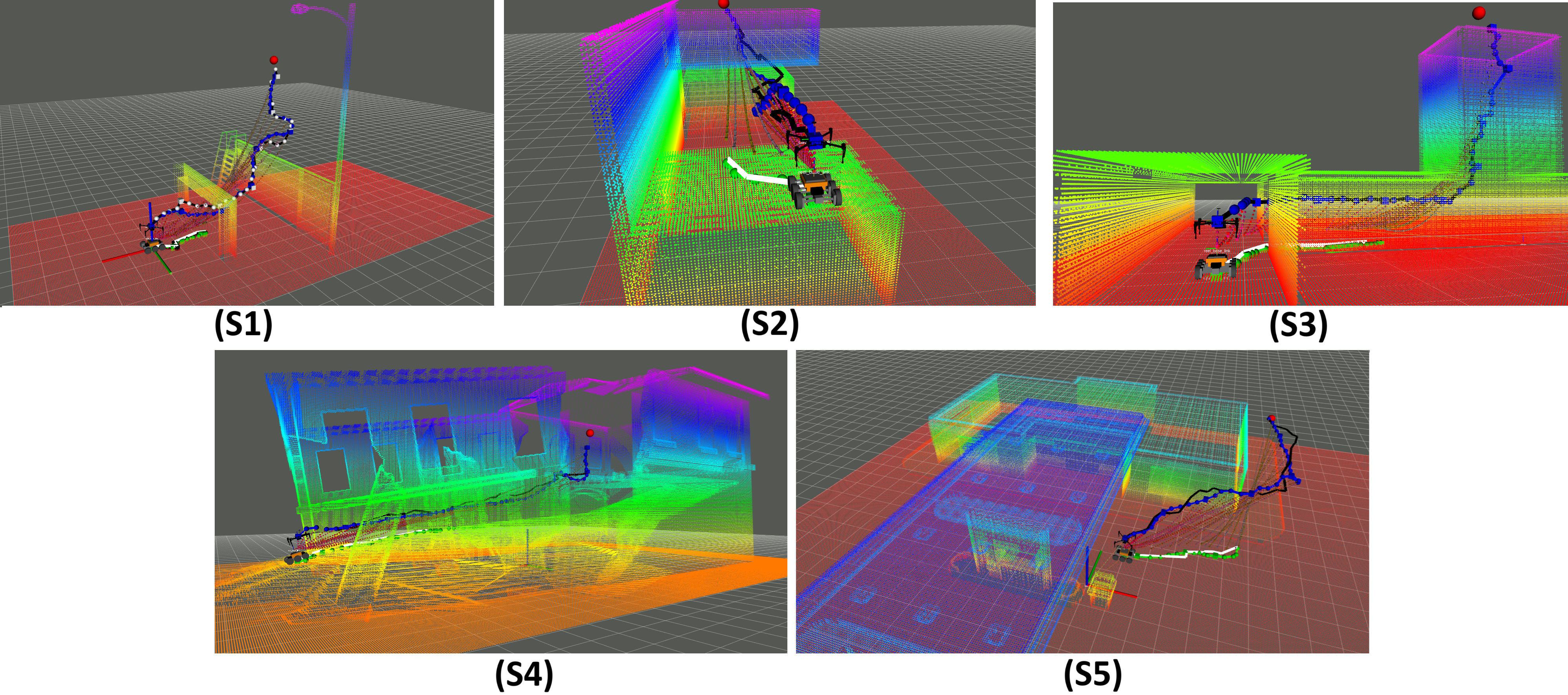}
  \caption{Scenarios considered for validation. S1: Open/constrained space with arc as obstacle. S2: Narrow/constrained space with denied access to UGV. S3: Confined space with outlet duct for UAV. S4: Collapsed Fire Station. S5: Open space gas station.}
  \label{fig:scenarios}
  \vspace*{-5mm}
\end{figure*}

\subsection{RRT* results}

In this section, we describe the results obtained by the RRT* path planner (Section \ref{sec:path_planning}) when using the catenary and parabola approaches to model the shape of the tether. 
The RRT* algorithm executes batches of 500 iterations, checking for a valid solution at the end of the batch. If not, another batch starts. The solution is used as an initial guess in the optimization step. We empirically set the maximum number of iterations to 10000, as we found that RRT* usually finds a solution before the 1000th iteration.  

First, we focus on the execution time required to solve a PDP or a CDP. Note that for a configuration to be valid in our proposed RRT* algorithm, a CDP or a PDP has to be solved between the UAV and UGV poses. In these experiments, we made the RRT* solve both problems, and we saved the execution times associated with each method. Figure \ref{fig:rrt_results} shows the violin plots of the execution times distributions for both Decision Problem methods (catenary and parabola). It is easy to see how the parabola method consistently outperforms the catenary method in mean and median time. Depending on the problem to be solved, the improvement in median times can range from a very noticeable 5x factor in S3, to also a significant 2x factor in S5.

Then, we evaluate the impact of each method on the execution time of the RRT* algorithm as a whole. As seen previously, our method for solving the PDP is faster than the method for solving the CDP. Therefore, the RRT* algorithm should also be faster when using the Parabola approach. Due to the random nature of the RRT* algorithm, we have repeated each experiment one hundred times. Table \ref{tab:results_rrt} shows the mean value and the standard deviation of the execution times of each method in each scenario. The table also shows the Decision Problem rate (DP rate), which stands for the percentage of solutions in which the Decision Problem must be computed, that is, there is no direct line of sight between the ground and the aerial robots. A scenario with DP rate of 0\% indicates that the Decision Problem was never computed, so ground and aerial configurations were able to be connected with a taut tether in all cases, while a scenario with DP rate of 100\% indicates that we needed to solve the Decision Problem for all ground and aerial configurations tested. As a rule of thumb, the higher the DP rate, the more complex/cluttered the scenario is.

We can see in Table \ref{tab:results_rrt} how the RRT* using PDP is clearly faster than using CPD in both mean time and standard deviation in all scenarios. The improvement depends on the type of problem to solve. The major improvement, 18x to 25x faster, is found in S1, which is relatively open and with obstacles in the middle of the path. S2, S3 and S4, cluttered closed volumes, are also solved very efficiently, with improvements ranging from 4x to 6x in time. Finally, S5 shows the smallest, yet significant, improvements of 1.9x to 2.8x faster. This is a large and very open scenario, which benefits solutions based on straight lines between robots (taut tether). In these cases, it is not necessary to compute the catenary and, therefore, the improvement of the PDP algorithm becomes less visible. These results match the DP rate, which is highest in S1 (greatest improvement), while S5 provides the lowest rate (smallest improvement).

In summary, the experimental results show that the proposed Parabola Decision Problem is a fast and convenient approximation to model a catenary in the RRT* algorithm.


\begin{figure*}[t!]
  \centering
  \includegraphics[width=0.141\linewidth]{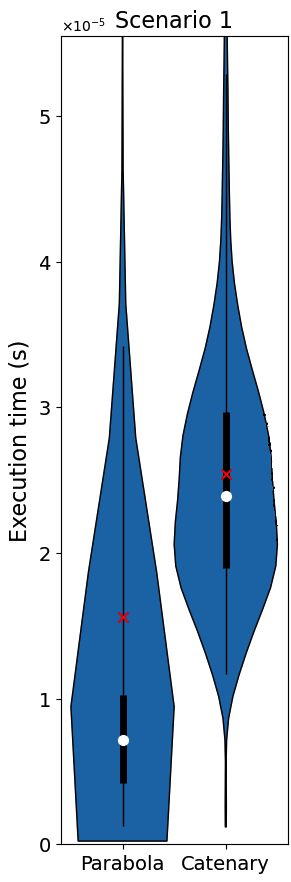}
  \hfill
  \includegraphics[width=0.15\linewidth]{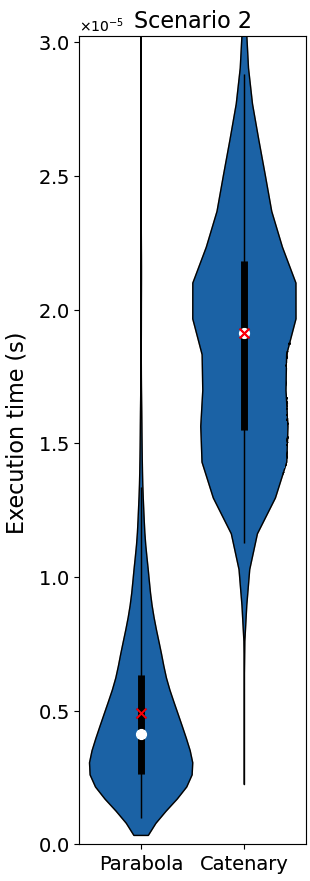}
  \hfill
  \includegraphics[width=0.15\linewidth]{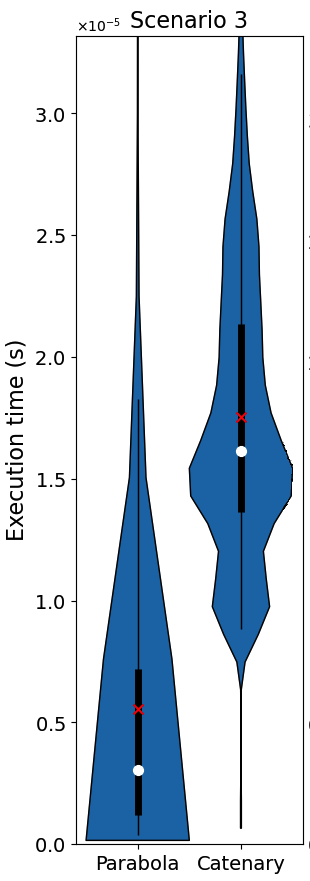}
  \hfill
  \includegraphics[width=0.151\linewidth]{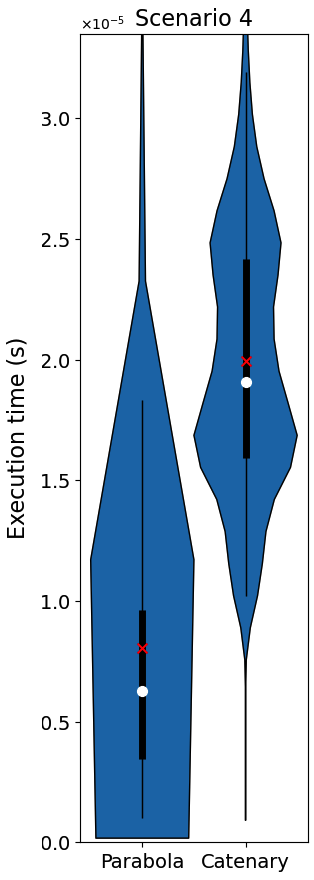}
  \hfill
  \includegraphics[width=0.141\linewidth]{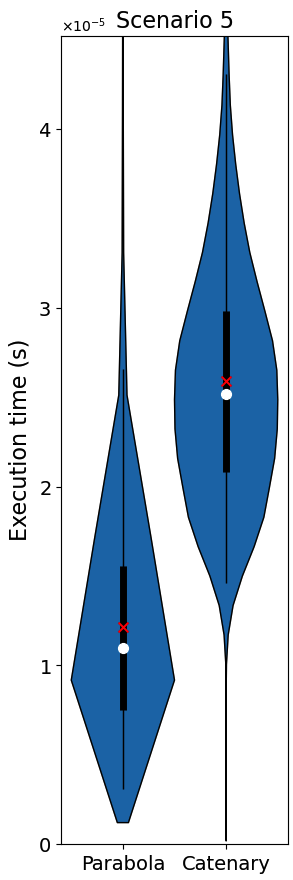}
  \caption{Violin plots of the distribution of the execution times of the Decision Problem test bench with Parabola and Catenary methods. In blue, the approximate shape of the distribution is represented. The mean value is represented as a red cross, the median as a white dot and the quartiles are linked with a black line. }
  \label{fig:rrt_results}
\end{figure*}


\begin{table}[t!]
	\caption{Execution time and decision problem rate results on RRT* }
 \label{tab:results_rrt}
	\begin{center}
		\begin{tabular}{|c|cc|c|}
			\hline
			\multirow{2}{*}{Scenario} & \multicolumn{2}{c|} {Execution Time (s)} & \multirow{2}{*}{DP rate (\%) } \\ \cline{2-3}
            & Catenary & Parabola &  \\
			\hline S1.1 & 5.16$\pm$6.66 & \textbf{0.28}$\pm$0.13 &93 \\
			\cline{1-4} S1.2 & 7.73$\pm$10.31 & \textbf{0.30}$\pm$0.10  & 91 \\
			\hline S2.1 & 0.31$\pm$0.36  & \textbf{0.06}$\pm$0.03  & 81 \\
			 \hline S2.2 & 0.35$\pm$0.39 &  \textbf{0.07}$\pm$0.03  & 80 \\
			\hline S3.1 & 3.08$\pm$1.56 & \textbf{0.47}$\pm$0.10 & 84 \\
			\hline S3.2 & 3.01$\pm$1.93 & \textbf{0.75}$\pm$0.26  & 86 \\
			 \hline S4.1  & 2.57$\pm$1.16 & \textbf{0.48}$\pm$0.19 & 81 \\
			\hline S4.2 & 2.43$\pm$2.20 & \textbf{0.44}$\pm$0.39  & 87 \\
			\hline S5.1 & 14.47$\pm$11.69 & \textbf{5.06}$\pm$2.75 & 69 \\
			\hline S5.2 & 10.74$\pm$5.46 & \textbf{5.56}$\pm$2.08 & 71 \\
			\hline
		\end{tabular}
	\end{center}
\vspace{-6mm}	
\end{table}



%

\subsection{Optimization results}

Different experiments were carried out to test the parabola approach within the optimization process. The results are presented in Table \ref{tab:experiments_new}. We benchmarked the new approaches with respect to \cite{smartinezr2023}.

We tested the methods in two different initial positions for each scenario. All experiments have been executed 100 times with the same set of parameters detailed below. The weight factors, empirically selected, are in the range $[0, 1]$.

\begin{itemize}
    \item Weighting factors for UGV: $\gamma_{eg}$ = 0.2, $\gamma_{og}$ = 0.08, $\gamma_{trav}$ = 0.5, $\gamma_{sg}$ = 0.12, $\gamma_{vg}$ = 0.05, $\gamma_{ag}$ = 0.005. 
    \item Weighting factors for UAV: $\gamma_{ea}$ = 0.25, $\gamma_{oa}$ = 0.08, $\gamma_{sa}$ = 0.14, $\gamma_{va}$ = 0.05, $\gamma_{aa}$ = 0.005. 
    \item Weighting factors for Tether: $\gamma_{ot}$ = 0.25, $\gamma_{u}$ = 0.1, $\gamma_{p}$ = 0.1
    \item Equidistance threshold: $\rho_{eg}$ and $\rho_{ea}$ are computed based on the initial path.   
    \item Collision threshold: $\rho_{oa}$ = 1.2, $\rho_{ot}$ = 0.1, $\rho_{og}$ = 1.2. 
    \item Traversability threshold: $\rho_{trav}$ = 0.001.
    \item Smoothness threshold: $\rho_{sg} $ = $\frac{\pi}{9}$, $\rho_{sa} $ = $\frac{\pi}{9}$.
    \item Desired velocity (m/s): $\rho_{vg}$ = 1.0, $\rho_{va}$ = 1.0.
    \item Desired velocity (m/s): $\rho_{ag}$ = 0.0, $\rho_{aa}$ = 0.0.
    \item Tether avoidance: $\beta$ = 10.0 .
\end{itemize}

The results of the experiments are summarized in Table \ref{tab:experiments_new}. Next paragraphs evaluate the different metrics:

\begin{table*}[t!]
\caption{Results of the optimizer for different parameterizations related to computing time, distance to obstacles, and trajectory time in simulated environments. SI: Scenario; M: Method (P = Parabola, C = Catenary, \cite{smartinezr2023})  F: Feasibility [\%]; T: Time for optimized solution [$s$];  DO: Distance to obstacles[$m$]; V: Velocity of the optimized trajectory [$m/s$]; A: Acceleration of the optimized trajectory [$m/s^2$]} 
\centering
\label{tab:experiments_new}
\footnotesize
\begin{adjustbox}{max width=\textwidth}
\begin{tabular}{|c|ccc|cc|cccccc|cccccc|}
\hline 
\multicolumn{4}{|c}{\textbf{ALL}} & 
\multicolumn{2}{|c}{\textbf{TETHER}} & 
\multicolumn{6}{|c|}{\textbf{UGV}}            & 
\multicolumn{6}{|c|}{\textbf{UAV}}            \\
\hline 
  \multicolumn{1}{|c}{\textbf{SI}} &
  \multicolumn{1}{|c}{\textbf{M}} &
  \multicolumn{1}{|c}{\textbf{F}} &
  \multicolumn{1}{|c}{\textbf{T}} &
  \multicolumn{1}{|c}{\textbf{\begin{tabular}[c]{@{}c@{}}Mean\\ DO\end{tabular}}} &
  \multicolumn{1}{|c}{\textbf{\begin{tabular}[c]{@{}c@{}}Min\\ DO\end{tabular}}} &
  \multicolumn{1}{|c}{\textbf{\begin{tabular}[c]{@{}c@{}}Mean\\ DO\end{tabular}}} &
  \multicolumn{1}{|c}{\textbf{\begin{tabular}[c]{@{}c@{}}Min\\ DO\end{tabular}}} &
  \multicolumn{1}{|c}{\textbf{\begin{tabular}[c]{@{}c@{}}Mean\\ V\end{tabular}}} &
  \multicolumn{1}{|c}{\textbf{\begin{tabular}[c]{@{}c@{}}Max\\ V\end{tabular}}} &
  \multicolumn{1}{|c}{\textbf{\begin{tabular}[c]{@{}c@{}}Mean\\ A\end{tabular}}} &
  \multicolumn{1}{|c}{\textbf{\begin{tabular}[c]{@{}c@{}}Max\\ A\end{tabular}}} &
  \multicolumn{1}{|c}{\textbf{\begin{tabular}[c]{@{}c@{}}Mean\\ DO\end{tabular}}} &
  \multicolumn{1}{|c}{\textbf{\begin{tabular}[c]{@{}c@{}}Min\\ DO\end{tabular}}} &
  \multicolumn{1}{|c}{\textbf{\begin{tabular}[c]{@{}c@{}}Mean\\ V\end{tabular}}} &
  \multicolumn{1}{|c}{\textbf{\begin{tabular}[c]{@{}c@{}}Max\\ V\end{tabular}}} &
  \multicolumn{1}{|c}{\textbf{\begin{tabular}[c]{@{}c@{}}Mean\\ A\end{tabular}}} &
  \multicolumn{1}{|c|}{\textbf{\begin{tabular}[c]{@{}c@{}}Max\\ A\end{tabular}}} \\
\hline 
\multirow{3}{*}{S1.1} & P & \textbf{98}  & 7.049  & 1.278 & \textbf{0.959} & 5.378  & \textbf{2.629} & 0.518 & 0.799 & 0.230  & 0.292  & 5.299  & \textbf{2.498} & 0.928 & 1.085 & 0.340 & 0.159  \\
& C & 85  & \textbf{6.666}  & 1.404 & 0.480 & 4.592  & 2.022 & 0.528 & 0.880 & 0.262  & 0.173  & 4.593  & 1.829 & 0.880 & 1.086  & 0.313  & 0.046  \\
 & \cite{smartinezr2023} & 86 & 465.9 & 1.14 & 0.25 & 2.28 & 1.55 & 0.36 & 1.07 & 0.02 & 0.81  & 1.87 & 0.80 & 0.97 & 1.19 & -0.002 & -0.09 \\ 
\hline 

\multirow{3}{*}{S1.2} & P & \textbf{97}  & \textbf{18.162} & 4.823 & \textbf{2.767} & 15.182 & 7.816 & 2.317 & 1.531 & 0.655  & 0.331  & 14.721 & 7.645 & 2.513 & 1.611 & 0.691 & 0.343  \\
 & C & 91  & 55.195 & 4.444 & 2.286 & 16.254 & \textbf{8.175} & 1.073 & 0.945 & 0.698  & 0.263  & 15.890 & \textbf{7.903} & 1.203 & 1.017  & 0.745  & 0.285  \\
 & \cite{smartinezr2023} & 80 & 529.7 & 1.31 & 0.16 & 3.76 & 1.60 & 0.43 & 1.25 & 0.02 & 0.78  & 2.21 & 0.89 & 0.91 & 1.19 & -0.002 & -0.08 \\
\hline 

\multirow{3}{*}{S2.1} & P & \textbf{100} & \textbf{1.495}  & 0.694 & \textbf{0.151} & 0.777  & 0.548 & 0.141 & 0.451 & 0.016  & 0.202  & 1.122  & \textbf{0.759} & 1.053 & 1.177 & 0.002 & -0.018 \\
 & C & \textbf{100} & 9.903  & 0.142 & 0.150 & 0.546  & \textbf{0.643} & 0.757 & 0.019 & 0.596  & 0.000  & 0.643  & 0.643 & 1.093 & -0.001 & -0.009 & 0.000  \\
 & \cite{smartinezr2023} & 95 & 123.6 & 0.72 & 0.10 & 0.69 & 0.56 & 0.13 & 0.91 & 0.02 & 0.96  & 1.19 & 0.76 & 1.01 & 1.15 & 0.002  & 0.01  \\
\hline 

\multirow{3}{*}{S2.2} & P & \textbf{100} & 1.188  & 0.735 & \textbf{0.152} & 1.156  & 1.003 & 0.082 & 0.576 & 0.019  & 0.554  & 1.225  & \textbf{0.782} & 1.026 & 1.119 & 0.001 & 0.013  \\
 & C & 99  & \textbf{0.371}  & 0.725 & 0.125 & 1.186  & \textbf{1.060} & 0.087 & 0.467 & 0.014  & 0.308  & 1.192  & 0.686 & 1.021 & 1.087  & -0.002 & -0.026 \\
 & \cite{smartinezr2023} &  95 & 109.3 & 0.75 & 0.13 & 0.75 & 0.58 & 0.19 & 1.00 & 0.03 & 1.09  & 1.16 & 0.67 & 0.99 & 1.09 & 0.000  & 0.07  \\
\hline 

\multirow{3}{*}{S3.1} & P & 92  & 1.493  & 0.645 & 0.060 & 1.139  & 0.673 & 0.541 & 1.124 & -0.003 & -0.140 & 1.077  & 0.638 & 1.154 & 1.495 & 0.016 & 0.039  \\
 & C & \textbf{95}  & \textbf{0.707}  & 0.647 & 0.103 & 1.143  & \textbf{0.754} & 0.550 & 1.041 & -0.004 & 0.039  & 1.031  & \textbf{0.657} & 1.123 & 1.382  & 0.014  & -0.036 \\
 & \cite{smartinezr2023} &  85 & 213.6 &  0.67 & \textbf{0.13} & 0.98 & 0.62 & 0.50 & 1.07 & 0.00 & 0.21  & 0.81 & 0.56 & 1.02 & 1.21 & 0.003  & 0.16  \\
\hline 

\multirow{3}{*}{S3.2} & P & \textbf{93}  & 1.271  & 0.711 & \textbf{0.157} & 1.416  & 0.622 & 0.583 & 1.133 & 0.018  & -0.169 & 1.360  & \textbf{0.658} & 1.080 & 1.348 & 0.041 & 0.039  \\
 & C & \textbf{93}  & \textbf{0.641}  & 0.694 & 0.112 & 1.138  & \textbf{0.685} & 0.617 & 1.146 & -0.002 & -0.031 & 1.030  & 0.633 & 1.090 & 1.343  & 0.016  & 0.054  \\
 & \cite{smartinezr2023} & 84 & 214.9  & 0.67 & 0.14 & 0.97 & 0.61 & 0.55 & 1.08 & 0.00 & -0.20 & 0.79 & 0.54 & 0.98 & 1.16 & 0.004  & 0.18  \\
\hline 

\multirow{3}{*}{S4.1} & P & \textbf{94}  & 1.495  & 0.666 & 0.078 & 1.786  & 1.055 & 0.713 & 1.122 & 0.011  & 0.129  & 1.227  & \textbf{0.777} & 1.110 & 1.242 & 0.000 & 0.029  \\
 & C & 55  & \textbf{0.851}  & 0.668 & 0.066 & 2.024  & 0.634 & 0.808 & 1.288 & 0.025  & 0.074  & 1.576  & 0.678 & 1.130 & 1.467  & 0.030  & -0.044 \\
 & \cite{smartinezr2023} & 83 & 539.0  & 0.74 & \textbf{0.14} & 2.12 & \textbf{1.40} & 0.45 & 1.14 & 0.02 & 1.00  & 1.16 & 0.65 & 0.94 & 1.12 & -0.002 & -0.06 \\
\hline 

\multirow{3}{*}{S4.2} & P & \textbf{99}  & 1.448  & 0.737 & \textbf{0.149} & 2.063  & 1.552 & 0.590 & 1.046 & 0.014  & 0.284  & 1.260  & \textbf{0.771} & 1.113 & 1.231 & 0.001 & 0.024  \\
 & C & 84  & \textbf{0.758}  & 0.751 & 0.131 & 2.501  & \textbf{1.574} & 0.614 & 1.115 & 0.037  & 0.199  & 1.719  & 0.733 & 1.085 & 1.254  & 0.037  & -0.043 \\
 & \cite{smartinezr2023} & 76 & 708.6 & 0.76 & 0.10 & 1.70 & 1.07 & 0.47 & 1.32 & 0.01 & 0.86  & 1.15 & 0.65 & 0.92 & 1.22 & -0.002 & -0.32 \\
\hline 

\multirow{3}{*}{S5.1} & P & \textbf{100} & 3.747  & 1.165 & \textbf{0.310} & 2.566  & \textbf{1.124} & 0.736 & 0.994 & 0.041  & 0.124  & 2.231  & \textbf{1.216} & 1.011 & 1.101 & 0.035 & -0.020 \\
 & C & 97  & \textbf{2.005}  & 1.025 & 0.283 & 3.293  & 0.858 & 0.888 & 1.783 & 0.146  & -0.327 & 2.948  & 1.017 & 1.044 & 1.561  & 0.165  & -0.300 \\
 & \cite{smartinezr2023} &  98 & 277.8 & 1.09 & 0.23 & 2.26 & 0.76 & 0.41 & 1.15 & 0.03 & 1.07  & 1.74 & 0.83 & 0.97 & 1.27 & -0.004 & -0.14 \\
\hline 

\multirow{3}{*}{S5.2} & P & \textbf{100} & 3.366  & 1.160 & \textbf{0.309} & 4.196  & \textbf{0.935} & 0.745 & 0.988 & 0.084  & 0.284  & 3.520  & \textbf{1.049} & 0.996 & 1.130 & 0.075 & -0.010 \\
 & C & 80  & \textbf{2.050}  & 1.039 & 0.223 & 4.184  & 0.794 & 0.875 & 1.613 & 0.115  & -0.065 & 3.455  & 0.966 & 1.049 & 1.507  & 0.103  & -0.025 \\
 & \cite{smartinezr2023} &  95 & 447.7  & 1.14 & 0.22 & 2.53 & 0.63 & 0.57 & 1.08 & 0.02 & 1.02  & 1.66 & 0.65 & 0.92 & 1.11 & -0.004 & -0.12\\

\hline 
\end{tabular}
\end{adjustbox}
\vspace*{-2mm}
\end{table*}

subsection{Feasibility}
The result of the planners is considered feasible (F) when the computed trajectory is collision-free for every agent of the system (UAV, UGV, tether). This means that $d^i_{og} > \rho_{og}$ and $d^i_{oa} > \rho_{oa}$ for every UGV and UAV state, respectively. In the case of the tether, $d^i_{ot,j} > \rho_{ot}$ for every sample in each tether configuration. In this context, we can see that the proposed solutions are consistently more feasible than \cite{smartinezr2023}, the average feasibility of the trajectory planner with the parabola approach is 97.3\%. Besides, with the catenary approach, the average feasibility is 87.9\%. In both, the unfeasible solutions are due to tether collisions. 

\subsubsection{Distance to obstacles}
Regarding the tether, the parabola approach provides the larger minimal distance from obstacles (Min DO) in all scenarios except S3.1 and S4.1, followed by the catenary and finally \cite{smartinezr2023}. In general terms, this means that the computed tethers with the parabola and catenary are safer in terms of distance to obstacles. With regard to the obstacle distance of the UGV and UAV, the safety values also improve. On average, the minimum values (Min DO) for UGV are 1.79m and 1.71m, for the parabola and catenary approaches, respectively. In the case of UAV, the average for the minimum value (Min DO) is 1.67m and 1.57m. Therefore, both methods behave similarly in terms of distance to obstacle. In general, the Parabola offers slightly better results than the Catenary, but the difference seems to be negligible. 

\subsubsection{Computation Time}
Regarding the computation time (T), we can see that the parabola and catenary approaches are two orders of magnitude faster on average than \cite{smartinezr2023}. This is mainly produced by the analytical parameterization of the tether in the optimizer (Section \ref{sec:implementation}), unlike the use of numerical methods to solve the transcendental catenary equation \cite{BOOKOFCURVES}. We use the results of \cite{smartinezr2023} as reported in the paper. 
We can also see that the parabola approach is slightly slower than the catenary approach in most scenarios. This delay comes from a larger number of iterations in the parabola solver, resulting in better solutions and with higher feasibility, as previously pointed out. 

\subsubsection{Velocities and Accelerations}
Table \ref{tab:experiments_new} shows the optimized velocities and accelerations (V, A). In general, both methods keep the velocity and acceleration values close to the desired values ($\rho_{vg}$ = 1.0, $\rho_{va}$ = 1.0, $\rho_{ag}$ = 0.0, $\rho_{aa}$ = 0.0).

In summary, the results provided by the new approaches clearly overcome the results presented in \cite{smartinezr2023}. We can see how the new approaches have similar or higher feasibility (F) in all scenarios, together with a clear reduction in computational time (T) of almost two orders of magnitude. This improvement is directly related to the use of analytical equations for the tether (parabola and catenary), which enables the optimizer to perform faster and more accurate gradients. The rest of the parameters, such as distance to obstacles, velocities, and accelerations, are of the same order in the three solutions, but the parabola and catenary improve consistently \cite{smartinezr2023}. 

Among the proposed approaches, Parabola seems to be the best solution, because it provides a clearly higher feasibility in almost all scenarios, while it offers similar capabilities in the rest of the parameters of the planned trajectory.

\section{Conclusions and Future Work}
\label{sec:conclusions}

The paper proposes a novel approach for the efficient parameterization and estimation of the state of a hanging tether for path and trajectory planning of a tethered marsupial robotic system combining an unmanned ground vehicle (UGV) and an unmanned aerial vehicle (UAV). The paper proposes integrating into the trajectory state to optimize the parameters that define the tether curve, and also demonstrates that the parabola curve approximation is a good and efficient representation of the tether.

Experimental results demonstrate that the approximation using a parameterization curve can generate smooth, collision-free trajectories in a fraction of the time taken by the state-of-the-art methods, thus improving the feasibility and efficiency of the obtained solutions. Furthermore, the implementation of the RRT* planning algorithm with the use of a parabola approach based on PDP improves the time calculation in complex three-dimensional environments, such as confined or obstacle-ridden spaces.

We notice that the parabola approximation is an effective alternative for trajectory planning in UAV-UGV systems, by significantly reducing the computation time without compromising the quality of the solutions. This opens new possibilities for the implementation of marsupial robotic systems in missions where real-time trajectory planning is required and in complex three-dimensional environments.

Future work will consider adapting the optimizer to incorporate local planning capabilities, enabling it to modify existing global planner solutions. This integration would provide greater flexibility in adjusting the trajectory as new information or obstacles arise. In addition, incorporating kinematical models of the UGV, UAV, and tether to allow for more realistic constraints, providing a closer representation of the physical capabilities and limitations of the system. 



\bibliographystyle{IEEEtran} 
\bibliography{IEEEabrv}
\end{document}